\documentclass[11pt]{article}
\usepackage[margin=1in]{geometry}
\usepackage{amsmath,amssymb,amsthm,bm,graphicx,booktabs,microtype}
\usepackage[numbers,sort&compress]{natbib}
\usepackage{xcolor}
\usepackage{placeins}
\usepackage{hyperref}
\hypersetup{colorlinks=true,linkcolor=black,citecolor=black,urlcolor=black}
\newtheorem{proposition}{Proposition}

\newcommand{\uale}{u_{\mathrm{ale}}}
\newcommand{\uepi}{u_{\mathrm{epi}}}
\newcommand{\umeas}{u_{\mathrm{meas}}}
\newcommand{\udiff}{u_{\mathrm{diff}}}
\newcommand{\Vpred}{V_{\mathrm{pred}}}
\newcommand{\abar}{\bar\alpha}
\newcommand{\Cond}{\mathcal{C}}

\title{PEEL-DDPM: Physics-Enabled Evidential Learning for the Denoising Diffusion Probabilistic Model}
\author{Ge Wang\\Biomedical Imaging Center, Rensselaer Polytechnic Institute\\Troy, New York, USA}

\begin{document}
\maketitle

\begin{abstract}
Normal--inverse-gamma (NIG) regression is not identifiable from its marginal Student-$t$ likelihood: the likelihood determines three combinations of four NIG parameters and leaves a one-dimensional $(\beta,\nu)$ fiber. This paper formulates physics-enabled evidential learning (PEEL) for a representative generative AI model commonly referred to as denoising diffusion probabilistic model (DDPM), denoted as PEEL-DDPM. A conventional measurement-conditioned DDPM is first trained only by $\varepsilon$-MSE and then frozen. A complete reverse trajectory produces a final reconstruction $\tilde x_0$, which is paired during training with the known object $x_0$ to form the physically interpretable image-domain residual $r=x_0-\tilde x_0$. A final-image evidential network fits a zero-location Student-$t$ model to this residual and learns only the identifiable coordinates $(\alpha,c)$, leaving the generated image itself untouched. Repeated scanner-noise realizations and repeated complete diffusion trajectories provide a nested Monte Carlo measurement of the final-image aleatoric variance, split into scanner-induced and sampler-induced components. Because the NIG fiber has exactly one missing scalar, the measured total $\uale=\umeas+\udiff$ uniquely identifies $(\beta,\nu)$ and yields $\Vpred=\umeas+\udiff+\uepi$. 
We give a sequential training design with no cross-loss coefficient. A feasibility experiment shows that the final-image formulation is learnable: on eight held-out objects, the Student-$t$ model attains empirical central-interval coverages of 49.3\%, 80.3\%, 90.1\%, and 95.6\% for nominal 50\%, 80\%, 90\%, and 95\% intervals, respectively; the mean squared residual is 0.967 times the mean predicted variance; and a single-image aleatoric head reaches pooled Spearman $\rho=0.785$ against an independent nested reference. A five-dose physical-channel stress test further gives a log--log dose slope of $-1.15$ for scanner-induced variance versus $-0.01$ for sampler-induced variance, while a higher-budget $96\times6$ nested reference attains median within-image split-half reliability $\rho=0.750$ across eight held-out objects.
\end{abstract}

\noindent\textbf{Keywords:}
Physics-enabled evidential learning (PEEL); denoising diffusion probabilistic model (DDPM);
normal-inverse-gamma regression; identifiability; aleatoric uncertainty; computed tomography (CT).

\section{Introduction}
\label{sec:intro}
Normal--inverse-gamma (NIG) regression \citep{amini2020deep} associates a predictive hierarchy with its output. After marginalization, however, the Student-$t$ likelihood depends on four NIG parameters only through three identifiable coordinates. The likelihood is therefore constant along a one-dimensional fiber, and the usual aleatoric/epistemic split is not identified. Evidence regularizers \citep{amini2020deep,meinert2023unreasonable}, reference priors \citep{malinin2018prior}, and latent hierarchies \citep{wang2026elvae} can select a representative on this fiber, but selection is not identification. PEEL \citep{wang2026peel} instead supplies an independent physical measurement of the missing coordinate and recovers $(\beta,\nu)$ algebraically.

Diffusion models create a tempting route to such a measurement because the forward process
\begin{equation}
x_t=\sqrt{\abar_t}\,x_0+\sqrt{1-\abar_t}\,\varepsilon,
\qquad \varepsilon\sim\mathcal N(0,I),
\label{eq:forward}
\end{equation}
is known and resamplable \citep{ho2020ddpm,song2021score}. 
The present work incorporates evidential learning by focusing on the error in the final reconstructed image. At training time the ground-truth object $x_0$ is known. After the complete reverse chain has produced $\tilde x_0$, the residual
\begin{equation}
r=x_0-\tilde x_0
\label{eq:imageerrorintro}
\end{equation}
is available directly, in image units, and already contains the cumulative effect of every reverse step. Therefore, it is used as the endpoint to which we attach the evidential model.

This distinction is important. Any \emph{single-step} $\varepsilon$ prediction has only an incremental contribution to the final $x_0$ estimate. By contrast, the final sample after the entire reverse path depends on the sequence of initialization, score estimates, and data conditioning. Its error reflects the collective effect of all per-step residuals. The proposed formulation therefore keeps $\varepsilon$-MSE exactly where it belongs---training the DDPM score---and moves Student-$t$ evidential learning to the final image.

Measurement conditioning supplies the physical structure needed by PEEL. We use a CT measurement
\begin{equation}
y=Ax_0+\eta,\qquad \eta\sim p_\eta(\cdot;x_0,I_0),
\label{eq:scan}
\end{equation}
where the scanner noise is object dependent because photon statistics depend on attenuation along each ray. While repeating the scan changes $\eta$, repeating the diffusion reconstruction changes the complete sampling trajectory. These are two known, resamplable sources of final-image variation. A nested law of total variance separates their contributions into scanner-induced and sampler-induced components. Their sum is used as the one independently measured scalar required to identify the NIG fiber.

Figure~\ref{fig:workflow} summarizes our overall design. The denoising diffusion probabilistic model
(DDPM) is trained conventionally and frozen. A second network fits $(\alpha,c)$ from final-image residuals. A third, sequentially trained head learns the two measured channel variances from the final reconstruction. The losses are never added, so no evidence regularizer or cross-loss coefficient is required. The result is a final uncertainty summary:
\[
\Vpred=\umeas+\udiff+\uepi,
\]
where the first term answers how much another scan can change the reconstruction, the second how much another diffusion trajectory can change it, and the third term is the NIG epistemic component under the PEEL operational definition.

This paper describes the formulation, a minimal feasibility experiment, a targeted multi-dose validation of the physical channel decomposition, and further research directions. Section~\ref{sec:method} develops the methodology, Section~\ref{sec:design} reports the pilot data, and Section~\ref{sec:disc} discusses relevant issues and concludes the paper.

\section{Methodology}
\label{sec:method}

\subsection{Conditional DDPM}
Let $x_0\in\mathbb R^{|\Omega|}$ be the object, $A$ the scanner forward operator, and $I_0$ the dose. We use the deterministic conditioning image
\begin{equation}
\Cond=\mathrm{FBP}(y),
\label{eq:condition}
\end{equation}
although any suitable function of $y$ could be substituted. A conditional DDPM predicts the forward noise by
\[
\hat\varepsilon_\theta=F_\theta(x_t,t,\Cond,I_0)
\]
and is trained by the standard loss
\begin{equation}
\mathcal L_{\mathrm{score}}(\theta)
=\mathbb E\big\|\varepsilon-F_\theta(x_t,t,\Cond,I_0)\big\|^2.
\label{eq:scoreloss}
\end{equation}
After training, $\theta=\theta^\star$ is frozen. Let $\xi$ denote all randomness of a complete reverse trajectory, including the initial state and ancestral Gaussian innovations. The final reconstruction is
\begin{equation}
\tilde x_0=G_{\theta^\star}(\Cond,I_0;\xi),
\label{eq:finalgenerator}
\end{equation}
where $G$ means the full sequence $x_T\rightarrow x_{T-1}\rightarrow\cdots\rightarrow x_0$.

\subsection{Final-image residual}
At a fixed diffusion level $t$, the usual one-step clean-image estimate is
\begin{equation}
\hat x_0^{(t)}=
\frac{x_t-\sqrt{1-\abar_t}\,\hat\varepsilon_\theta}{\sqrt{\abar_t}}.
\label{eq:onestepx0}
\end{equation}
Substituting Eq.~\eqref{eq:forward} gives
\begin{equation}
\hat x_0^{(t)}-x_0
=\sqrt{\frac{1-\abar_t}{\abar_t}}\,
(\varepsilon-\hat\varepsilon_\theta).
\label{eq:affineequiv}
\end{equation}
Thus, replacing the per-step noise residual by the corresponding one-step image residual changes only a known scale factor. It does not create a new statistical target.

The final reconstruction in Eq.~\eqref{eq:finalgenerator} is different. It depends recursively on all reverse transitions and on the measurement condition. Let us define the final-image residual
\begin{equation}
r=x_0-\tilde x_0.
\label{eq:finalresid}
\end{equation}
There is no single scalar $a_t$ such that $r=a_t(\varepsilon-\hat\varepsilon_t)$ for one selected step. The complete path is already integrated into $\tilde x_0$. We therefore use Eq.~\eqref{eq:finalresid}, not Eq.~\eqref{eq:affineequiv}, as the evidential target.
Note that setting $t=0$ does not recover the final-image residual either: by definition,
$\bar{\alpha}_0=1$, so the one-step estimator collapses trivially to the
ground-truth image $x_0$, yielding zero residual rather than
$x_0-\tilde{x}_0$.

\subsection{Final-image NIG model}
At each pixel $i$, the generated value $\tilde x_{0,i}$ is the conventional output of the DDPM model. We write the NIG hierarchy
\begin{align}
\sigma_i^2 &\sim \mathrm{InvGamma}(\alpha_i,\beta_i),\\
\mu_i\mid\sigma_i^2 &\sim \mathcal N\!\left(\tilde x_{0,i},\frac{\sigma_i^2}{\nu_i}\right),\\
x_{0,i}\mid\mu_i,\sigma_i^2 &\sim \mathcal N(\mu_i,\sigma_i^2),
\label{eq:nig}
\end{align}
with $\alpha_i>1$, $\beta_i>0$, and $\nu_i>0$. Marginally,
\begin{equation}
r_i=x_{0,i}-\tilde x_{0,i}
\sim t_{2\alpha_i}\!\left(0,\frac{c_i}{\alpha_i}\right),
\qquad
c_i=\beta_i\left(1+\frac1{\nu_i}\right).
\label{eq:marginal}
\end{equation}
For simplicity, we assume negligible bias in the final-image estimate, so that the residual distribution is centered at zero.
The associated variances are
\begin{equation}
V_{\mathrm{pred},i}=\frac{c_i}{\alpha_i-1},\qquad
u_{\mathrm{ale},i}=\frac{\beta_i}{\alpha_i-1},\qquad
u_{\mathrm{epi},i}=\frac{\beta_i}{\nu_i(\alpha_i-1)},
\qquad V_{\mathrm{pred},i}=u_{\mathrm{ale},i}+u_{\mathrm{epi},i}.
\label{eq:decomp}
\end{equation}

\begin{proposition}[Final-image NIG fiber]
\label{prop:fiber}
With the location fixed to $\tilde x_{0,i}$, the marginal Student-$t$ likelihood depends on $(\beta_i,\nu_i)$ only through $c_i$. Hence, it is constant along
\begin{equation}
\beta_i(\nu_i)=\frac{c_i\nu_i}{1+\nu_i},\qquad \nu_i>0.
\label{eq:fiber}
\end{equation}
The final-image likelihood therefore identifies $(\alpha_i,c_i)$ but not the aleatoric/epistemic split.
\end{proposition}
This formulation has an important practical consequence: Student-$t$ NLL is never allowed to move the generated image. The DDPM output $\tilde x_0$ is fixed, and the NLL learns only the dispersion coordinates $(\alpha,c)$ from the real image residual. 

\subsection{Final-image channel decomposition}
Fix $x_0$ and $I_0$. Draw a scanner realization $\eta$, form $\Cond(\eta)$, and then draw a complete reverse trajectory $\xi$. The final residual is
\[
r_i(\eta,\xi)=x_{0,i}-G_{\theta^\star,i}(\Cond(\eta),I_0;\xi).
\]
Because $x_0$ is constant within this repeated experiment,
\begin{equation}
\mathrm{Var}_{\eta,\xi}[r_i\mid x_0]
=\mathrm{Var}_{\eta,\xi}[\tilde x_{0,i}\mid x_0].
\label{eq:erroroutputvar}
\end{equation}
Thus, the Monte Carlo teacher can be interpreted equivalently as variance of the final image or variance of the final image \emph{error}.

\begin{proposition}[Law of total variance]
\label{prop:nested}
For fixed $(x_0,I_0)$ and pixel $i$,
\begin{equation}
\underbrace{\mathrm{Var}_{\eta,\xi}[r_i]}_{u_{\mathrm{ale},i}}
=
\underbrace{\mathbb E_\eta\!\left[\mathrm{Var}_{\xi\mid\eta}[r_i]\right]}_{u_{\mathrm{diff},i}\ \text{(sampler-induced)}}
+
\underbrace{\mathrm{Var}_\eta\!\left[\mathbb E_{\xi\mid\eta}[r_i]\right]}_{u_{\mathrm{meas},i}\ \text{(scanner-induced)}}.
\label{eq:nested}
\end{equation}
\end{proposition}
\begin{proof}
Apply the law of total variance to $r_i$ conditioning on the scanner realization $\eta$.
\end{proof}

We define $\uale=\umeas+\udiff$ operationally as the final-image variance caused by known, resamplable channels. The scanner component is the part that can change under a repeat acquisition of the same object, while the sampler component is the part that can change when the same scan is reconstructed with another diffusion trajectory.
The final residual also admits the exact bias--variance identity
\begin{equation}
\mathbb E_{\eta,\xi}[r_i^2\mid x_0]
=b_i(x_0)^2+u_{\mathrm{ale},i},
\qquad
b_i(x_0)=\mathbb E_{\eta,\xi}[r_i\mid x_0].
\label{eq:biasvariance}
\end{equation}
This separates systematic reconstruction bias from stochastic channel uncertainty. $\uale$ is not claimed to equal the full mean squared error when $b_i\neq0$.

\subsection{Unbiased nested estimators}
Figure~\ref{fig:teacher} shows the final-image teacher. Draw $R_{\mathrm m}$ independent scanner realizations $\eta^{(m)}$. For each scan, draw $R_{\mathrm d}$ independent complete reverse trajectories $\xi^{(m,d)}$ and compute
\[
\tilde x_0^{(m,d)}=G_{\theta^\star}(\Cond^{(m)},I_0;\xi^{(m,d)}),
\qquad
r^{(m,d)}=x_0-\tilde x_0^{(m,d)}.
\]
With row means $\bar r^{(m)}=R_{\mathrm d}^{-1}\sum_d r^{(m,d)}$ and grand mean $\bar r=R_{\mathrm m}^{-1}\sum_m\bar r^{(m)}$, let us define
\begin{equation}
 u^{\mathrm{MC}}_{\mathrm{diff},i}
 =\frac{1}{R_{\mathrm m}(R_{\mathrm d}-1)}
 \sum_{m=1}^{R_{\mathrm m}}\sum_{d=1}^{R_{\mathrm d}}
 \big(r_i^{(m,d)}-\bar r_i^{(m)}\big)^2 
\label{eq:udiff}
\end{equation}
and
\begin{equation}
 u^{\mathrm{MC}}_{\mathrm{meas},i}
 =\frac{1}{R_{\mathrm m}-1}\sum_{m=1}^{R_{\mathrm m}}
 \big(\bar r_i^{(m)}-\bar r_i\big)^2
 -\frac{u^{\mathrm{MC}}_{\mathrm{diff},i}}{R_{\mathrm d}},
\qquad
u^{\mathrm{MC}}_{\mathrm{ale},i}=u^{\mathrm{MC}}_{\mathrm{meas},i}+u^{\mathrm{MC}}_{\mathrm{diff},i}.
\label{eq:umeas}
\end{equation}

\begin{proposition}[Unbiasedness]
\label{prop:unbiased}
Equations~\eqref{eq:udiff}--\eqref{eq:umeas} are unbiased for the two components of Eq.~\eqref{eq:nested}.
\end{proposition}
\begin{proof}
Equation~\eqref{eq:udiff} averages unbiased within-scan sample variances. The variance of each row mean is $\umeas+\udiff/R_{\mathrm d}$, so subtracting $u^{\mathrm{MC}}_{\mathrm{diff}}/R_{\mathrm d}$ from the between-row sample variance removes the finite-$R_{\mathrm d}$ contribution.
\end{proof}

Small negative estimates of $u^{\mathrm{MC}}_{\mathrm{meas}}$ could occur due to finite-sample Monte Carlo fluctuation. For nonnegative training targets we clip these values at zero and report the clipping frequency, since a high frequency is treated as evidence that $R_{\mathrm d}$ or $R_{\mathrm m}$ is inadequate.

\begin{figure}[t]
\centering
\includegraphics[width=\linewidth]{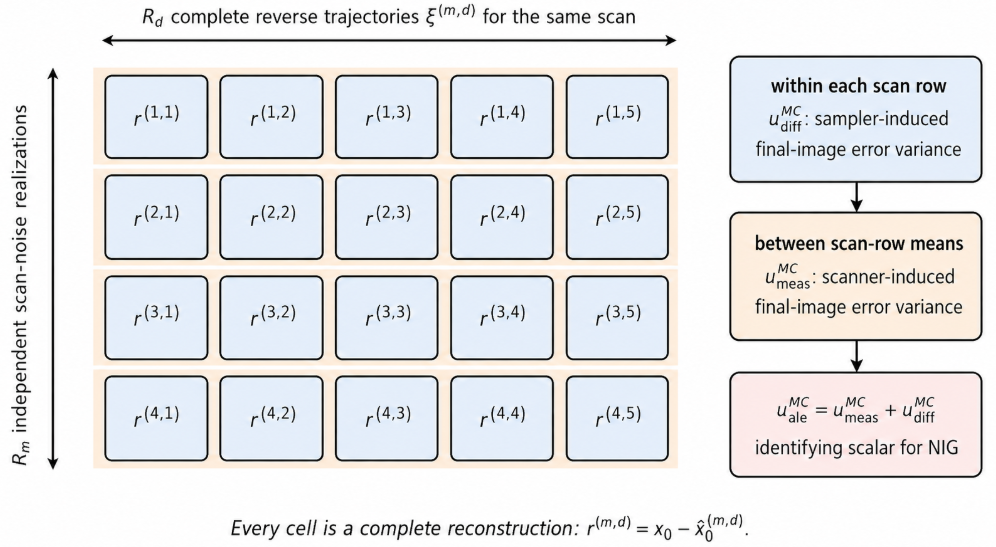}
\caption{Nested teacher for estimation of aleatoric uncertainty. Each row fixes one scan and varies the entire diffusion trajectory, while different rows use independent scans of the same object. Every cell is a complete final reconstruction and therefore has an image residual $r^{(m,d)}=x_0-\tilde x_0^{(m,d)}$. Within-row variation estimates $\udiff$ and corrected between-row variation estimates $\umeas$.}
\label{fig:teacher}
\end{figure}

\subsection{Image NIG identification}
\begin{proposition}[Identification]
\label{prop:recover}
If $(\alpha_i,c_i)$ and an independently obtained $u_{\mathrm{ale},i}$ are known, with
\begin{equation}
0<u_{\mathrm{ale},i}<\frac{c_i}{\alpha_i-1},
\label{eq:admissible}
\end{equation}
then the NIG representative on the fiber of Eq.~\eqref{eq:fiber} is unique:
\begin{equation}
\beta_i=(\alpha_i-1)u_{\mathrm{ale},i},
\qquad
\nu_i=\frac{u_{\mathrm{ale},i}}{c_i/(\alpha_i-1)-u_{\mathrm{ale},i}}.
\label{eq:recover}
\end{equation}
\end{proposition}
No KL term, reference prior, or evidence regularizer appears in Eq.~\eqref{eq:recover}. The two measured components $\umeas$ and $\udiff$ provide more information than the one-dimensional NIG fiber, and their sum identifies the NIG split.

\subsection{Sequential training and inference}
\label{sec:stages}
Figure~\ref{fig:workflow} illustrates the training and inference processes respectively. These can be also presented as four stages in a sequential order. Further details are as follows.
\begin{figure}[th]
\centering
\includegraphics[width=\linewidth]{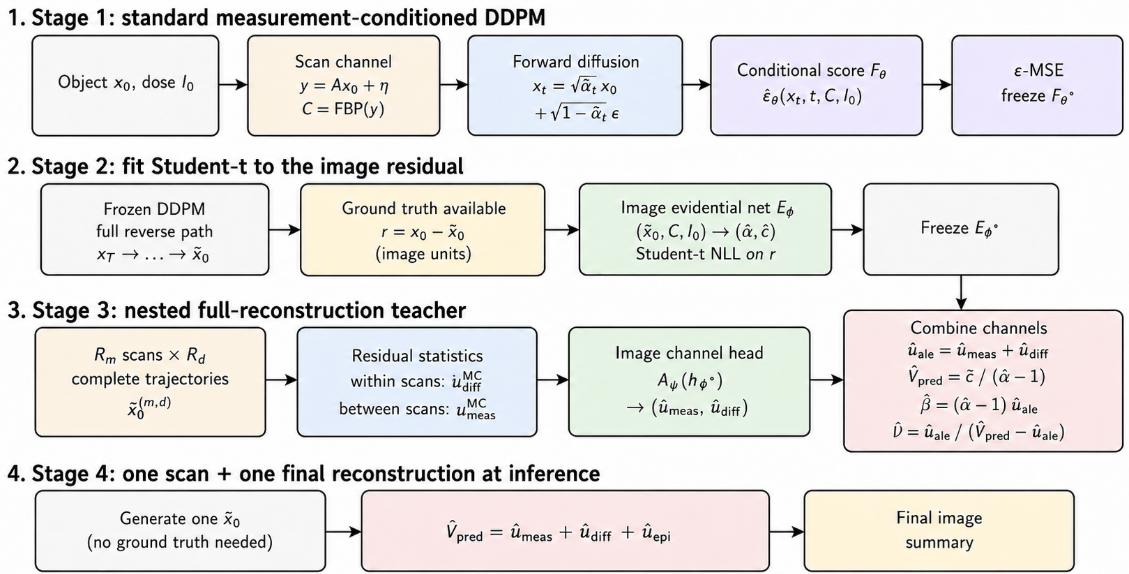}
\caption{PEEL-DDPM workflow. Stage~1 trains a conventional measurement-conditioned DDPM by $\varepsilon$-MSE and freezes it. Stage~2 runs complete reverse trajectories and fits the identifiable Student-$t$ coordinates $(\alpha,c)$ to the image residual $r=x_0-\tilde x_0$. Stage~3 resamples both the scanner and the complete diffusion trajectory to measure $\umeas$ and $\udiff$, then trains an image channel head. Stage~4 combines the learned likelihood coordinates with $\uale=\umeas+\udiff$ to recover $(\beta,\nu)$ and report an image uncertainty summary.}
\label{fig:workflow}
\end{figure}

\paragraph{Stage 1: standard conditional DDPM.}
Train $F_\theta$ only by Eq.~\eqref{eq:scoreloss}, using random $t$ and forward noise $\varepsilon$, and freeze $\theta^\star$. No evidential loss is present in the DDPM training.

\paragraph{Stage 2: Student-$t$ fitting of final image residuals.}
For a training object and scan, sample a complete trajectory and obtain $\tilde x_0$. A final-image evidential network $E_\phi$ receives $(\tilde x_0,\Cond,I_0)$ and returns shared features $h_\phi$ and
\begin{equation}
\hat\alpha=1+\operatorname{softplus}(W_\alpha\!\ast h_\phi)+\epsilon_\alpha,
\qquad
\hat c=\hat\alpha\big[\operatorname{softplus}(W_c\!\ast h_\phi)+\epsilon_0\big],
\label{eq:heads}
\end{equation}
where $\epsilon_\alpha=10^{-4}$ and $\epsilon_0=10^{-8}$. With the actual image residual $r=x_0-\tilde x_0$, the pixelwise NLL is
\begin{equation}
\ell^{\mathrm{NLL}}
=\log\Gamma(\hat\alpha)-\log\Gamma\!\left(\hat\alpha+\tfrac12\right)
+\tfrac12\log(2\pi\hat c)
+\left(\hat\alpha+\tfrac12\right)
\log\!\left(1+\frac{r^2}{2\hat c}\right).
\label{eq:nllpixel}
\end{equation}
Only $(\alpha,c)$ are learned, since the location is the already generated image. After convergence, $E_{\phi^\star}$ is frozen.

\paragraph{Stage 3: nested physical teacher and final-image aleatoric head.}
Build $u^{\mathrm{MC}}_{\mathrm{meas}}$ and $u^{\mathrm{MC}}_{\mathrm{diff}}$ by Eqs.~\eqref{eq:udiff}--\eqref{eq:umeas}, and form $u^{\mathrm{MC}}_{\mathrm{ale}}=u^{\mathrm{MC}}_{\mathrm{meas}}+u^{\mathrm{MC}}_{\mathrm{diff}}$. For a single scan of an object and the resultant final reconstruction, we compute frozen features $h_{\phi^\star}$ and train the positive identifying head
\[
\hat u_{\mathrm{ale}}=A_\psi(h_{\phi^\star})>0.
\]
The primary loss is
\begin{equation}
\mathcal J_{\mathrm{ale}}(\psi)=
\mathbb E\frac1{|\Omega|}\sum_{i\in\Omega}
\left[
\log(\hat u_{\mathrm{ale},i}+\epsilon_0)-
\log(u^{\mathrm{MC}}_{\mathrm{ale},i}+\epsilon_0)
\right]^2.
\label{eq:chloss}
\end{equation}
Only $\psi$ is updated. Auxiliary heads may be trained against the two measured components separately when a scanner/sampler decomposition is desired. The minimal feasibility experiment below evaluates only the total identifying coordinate $\uale$. The teacher repeats are never available to the predictor at inference, because many realizations teach the mapping for a fresh realization to use the mapping.

\paragraph{Stage 4: single-reconstruction inference.}
Given one scan, run one full reverse trajectory to obtain $\tilde x_0$. Evaluate $(\hat\alpha,\hat c)$ and $\hat\uale$ from the final image, and form
\begin{equation}
\hat\Vpred=\frac{\hat c}{\hat\alpha-1},
\qquad
\hat\uepi=\hat\Vpred-\hat\uale.
\label{eq:infer1}
\end{equation}
Where Eq.~\eqref{eq:admissible} holds, we can recover $(\hat\beta,\hat\nu)$ by Eq.~\eqref{eq:recover}. The minimal identified report is therefore
\begin{equation}
\hat\Vpred=\hat\uale+\hat\uepi,
\qquad
\uale=\umeas+\udiff
\label{eq:report}
\end{equation}
at the final reconstructed image. If auxiliary heads are trained, $\hat\uale$ may additionally be displayed as the repeat-scan and repeat-sampler contributions, but the NIG identification itself requires only their total.

\section{Experimental Design and Results}
\label{sec:design}
We first report one deliberately reduced experiment whose sole purpose is to test whether our PEEL-DDPM formulation is computationally and statistically viable. It is not the full validation study, since it uses one dose, fewer objects, and a shorter diffusion chain, and its independent Monte Carlo reference is also preliminary.

\subsection{Experimental design}
\label{sec:pilot}

The object family matches PEEL \citep{wang2026peel} but only the middle dose $I_0=3.50\times10^4$. Each $32\times32$ object contains a random body ellipse with attenuation in $[0.017,0.022]$, four to seven internal ellipses with signed contrast in $[-0.005,0.006]$, and one to three small high-contrast lesions with added attenuation in $[0.006,0.011]$, clipped to $[0,0.035]$ and mapped affinely to $[-1,1]$. Objects, not noise or trajectory realizations, are separated across splits.

Images were $32\times32$ with 36 parallel-beam views. We used 360 training objects, 40 validation objects, and eight held-out test objects. To keep the experiment small, the DDPM used $T=30$ cosine-schedule steps. The score model was trained for 1000 optimization iterations and then frozen. The final-image Student-$t$ network was trained from two independent final reconstructions per training/validation object for 800 iterations. The single-image aleatoric head was trained for 2000 iterations against nested teachers using $R_{\mathrm m}\times R_{\mathrm d}=12\times4$ complete reconstructions per teacher object. Test references were generated independently using $24\times6$ complete reconstructions per held-out object. All residual-modeling quantities were computed after affine normalization of the image range to $[-1,1]$. The reconstruction RMSE below is reported in the original attenuation units.

The measurement model is Poisson counts $\mathrm{Poisson}(I_0e^{-(Ax_0)_j})$ plus additive Gaussian electronic noise of standard deviation $\sigma_e$, floored at $0.5$ counts before the logarithm. Stage~1 uses AdamW with batch size 64, initial learning rate $10^{-3}$ with cosine decay, weight decay $10^{-5}$, and gradient-norm clipping at 5. Stages~2 and 3 use the same optimizer with independently selected validation checkpoints. Final reconstructions used in Stage~2 are generated only after Stage~1 is frozen. Teacher and evaluation trajectories are disjoint from those used to train the Student-$t$ head.

\subsection{Representative results}

The first feasibility question is whether the Student-$t$ model fitted to the actual final-image residual behaves like a calibrated image-error distribution on unseen objects. It did in this pilot: empirical coverage tracked nominal coverage closely from 50\% through 95\%, and the mean squared held-out residual was 96.7\% of the mean predicted variance. Thus, moving the evidential endpoint from $\varepsilon-\hat\varepsilon$ to $x_0-\tilde x_0$ not only produced a more interpretable variable but also yielded a quantitatively coherent predictive distribution for the actual reconstruction error.

The second question is whether the physically measured final-image aleatoric field can be inferred from one scan and one complete reverse trajectory. The one-image head attained pooled $\rho=0.785$ and median within-image $\rho=0.723$ against the independent $24\times6$ nested reference. The field was also relevant to the realized image error: pixels in the highest predicted-$\uale$ quintile had 3.82 times the mean squared residual of pixels in the lowest quintile, while the corresponding ratio for total Student-$t$ predictive variance was 4.89. Figure~\ref{fig:pilotresult} shows a representative held-out case selected by within-image correlation nearest the median. Fig.~\ref{fig:pilotcal} shows interval calibration and error stratification.
\begin{figure}[th]
\centering
\includegraphics[width=\linewidth]{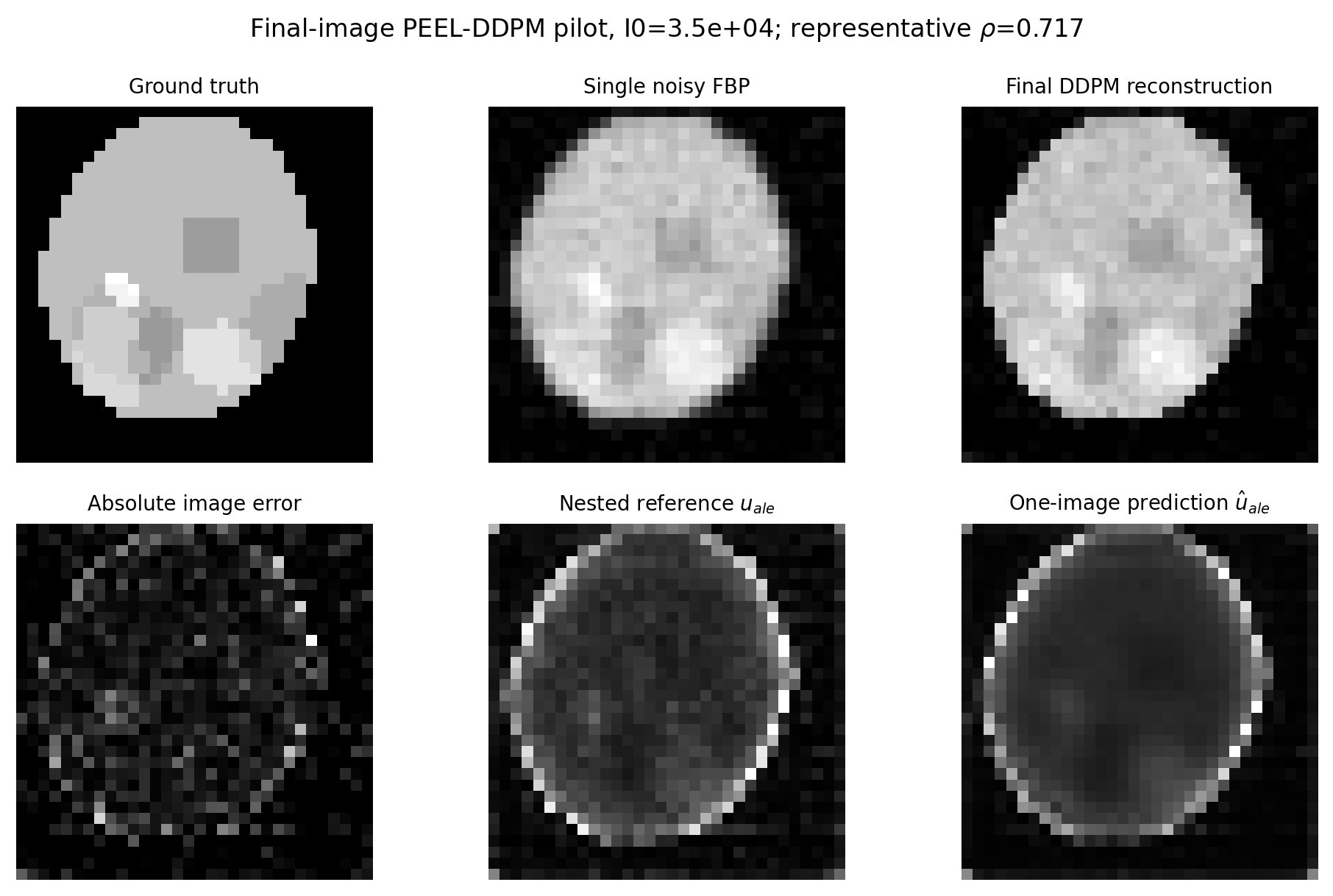}
\caption{Representative held-out final-image feasibility case. The nested $u_{\rm ale}$ reference is computed from independent repeated scans and complete reverse trajectories. The prediction uses only one scan and one final reconstruction. The displayed case has within-image Spearman correlation $\rho=0.717$, close to the test-set median. The uncertainty panels share common display limits.}
\label{fig:pilotresult}
\end{figure}
\begin{figure}[th]
\centering
\includegraphics[width=\linewidth]{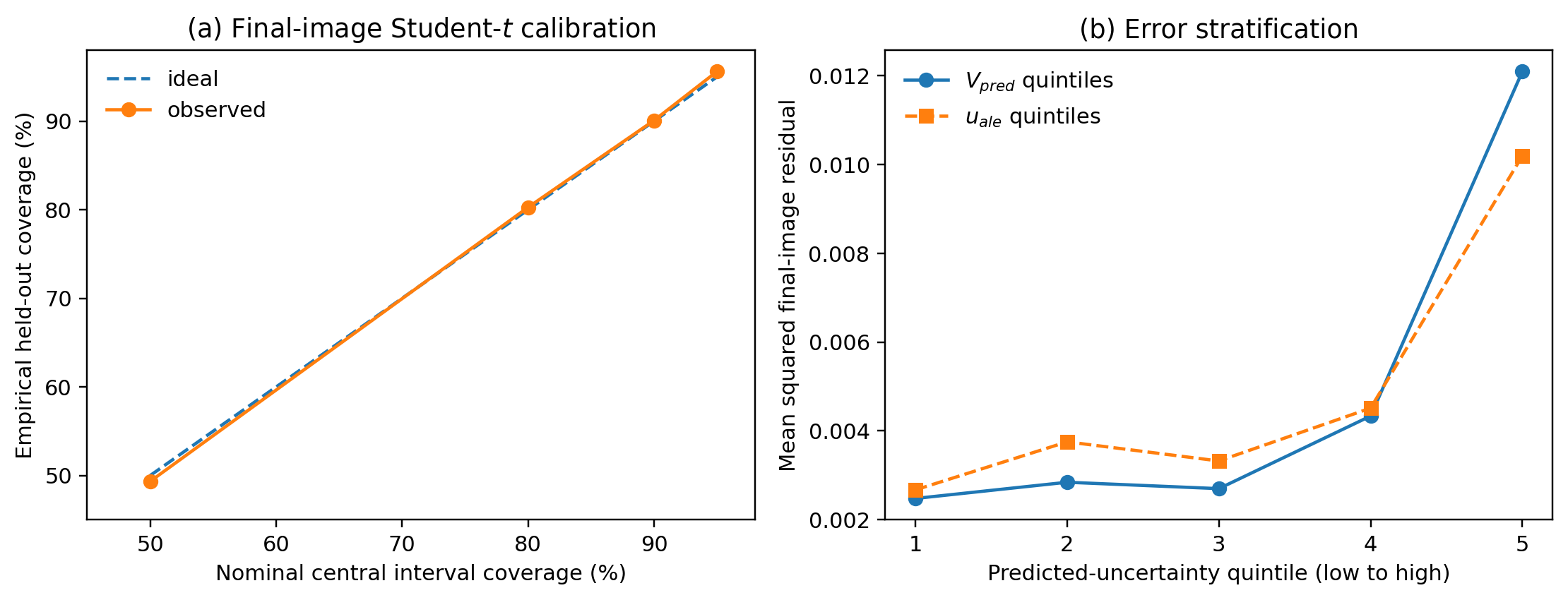}
\caption{Held-out image-error behavior in the feasibility study. (a) Empirical coverage of central Student-$t$ intervals closely follows nominal coverage, and (b) mean squared final-image residual increases strongly in the highest uncertainty quintile. The highest-to-lowest ratios are 4.89 for total predictive variance and 3.82 for the independently trained aleatoric coordinate.}
\label{fig:pilotcal}
\end{figure}

The nested reference remains the limiting factor. Its pooled split-half correlation was only $0.497$. This means that the $u_{\rm ale}$ correlations in Table~\ref{tab:pilot} should not be interpreted as final estimates. At the same time, the reference was not spatially structureless: the median within-image coefficients of variation were $0.384$ for $\umeas$ and $0.408$ for $\udiff$, the mean scanner fraction was $40.3\%$, and only $3.7\%$ of raw scanner-component estimates were negative before clipping. Taken together, these observations support feasibility while motivating a larger repeated-reconstruction budget for definitive validation.
\begin{table}[h]
\centering
\caption{One-dose feasibility experiment for the final-image formulation. Correlations against the nested reference are exploratory, with split-half reliability of $0.497$.}
\label{tab:pilot}
\begin{tabular}{lc}
\toprule
Quantity & Pilot result\\
\midrule
Final reconstruction RMSE & $1.224\times10^{-3}\pm6.32\times10^{-5}$\\
Final reconstruction SSIM & $0.891\pm0.015$\\
Student-$t$ coverage, nominal $50/80/90/95\%$ & $49.3/80.3/90.1/95.6\%$\\
$\mathbb E[r^2]/\mathbb E[\hat V_{\rm pred}]$ & $0.967$\\
$\hat u_{\rm ale}$ pooled Spearman $\rho$ & $0.785$\\
Median within-image $\rho$ (range) & $0.723$ ($0.504$--$0.852$)\\
Case-mean $\rho$ & $0.976$\\
Nested-reference split-half $\rho$ & $0.497$\\
Admissible fraction $0<\hat u_{\rm ale}<\hat V_{\rm pred}$ & $99.30\%$\\
Highest/lowest uncertainty-quintile MSE ratio, $\hat u_{\rm ale}$ & $3.82$\\
Highest/lowest uncertainty-quintile MSE ratio, $\hat V_{\rm pred}$ & $4.89$\\
Mean scanner fraction $\umeas/(\umeas+\udiff)$ & $0.403$\\
\bottomrule
\end{tabular}
\end{table}
\subsection{Sizing the repeated final-reconstruction}
The final-image formulation is more expensive than a per-step teacher because each Monte Carlo cell requires a complete reverse trajectory. We therefore size the repeat budget empirically rather than fixing it by convenience. A pilot on five held-out objects at each dose starts with $R_{\mathrm m}=16$ scans and $R_{\mathrm d}=4$ trajectories per scan. For an ordinary variance estimate from $R$ independent realizations, the Gaussian relative standard error is approximately
\begin{equation}
f(R)=\sqrt{\frac{2}{R-1}},
\label{eq:rse}
\end{equation}
which provides a first-order noise floor for the total final-image variance. The nested components are assessed more directly by bootstrap resampling of scan rows and trajectories within rows.

\FloatBarrier
\subsection{Additional multi-dose validation}
\label{sec:multidose}

We next tested the physical channel decomposition across the five planned dose levels
$I_0\in\{1.50,2.25,3.50,5.25,8.00\}\times10^4$ using the same eight held-out objects.
To isolate whether the decomposition itself responds to measurement physics, the DDPM trained at the middle dose was frozen and applied without retraining.  At each dose and object, an independent $R_{\mathrm m}\times R_{\mathrm d}=24\times6$ nested experiment was generated.  Thus, this experiment is a dose-stress test of the measured channels, rather than a fully retrained multi-dose prediction study.

The scanner-induced component decreased strongly with dose, whereas the sampler-induced component remained nearly constant (Table~\ref{tab:dose}).  A log--log fit gave slope $-1.149$ ($R^2=0.999$) for $u_{\mathrm{meas}}$ and slope $-0.006$ for $u_{\mathrm{diff}}$.  The mean scanner fraction $u_{\mathrm{meas}}/(u_{\mathrm{meas}}+u_{\mathrm{diff}})$ consequently decreased from $0.650$ at $I_0=1.50\times10^4$ to $0.215$ at $I_0=8.00\times10^4$.  This separation is consistent with the intended interpretation: acquisition noise follows dose, whereas reverse-process sampling variability is largely algorithmic and dose-insensitive in this frozen-model test.

\begin{figure}[h]
\centering
\includegraphics[width=0.58\linewidth]{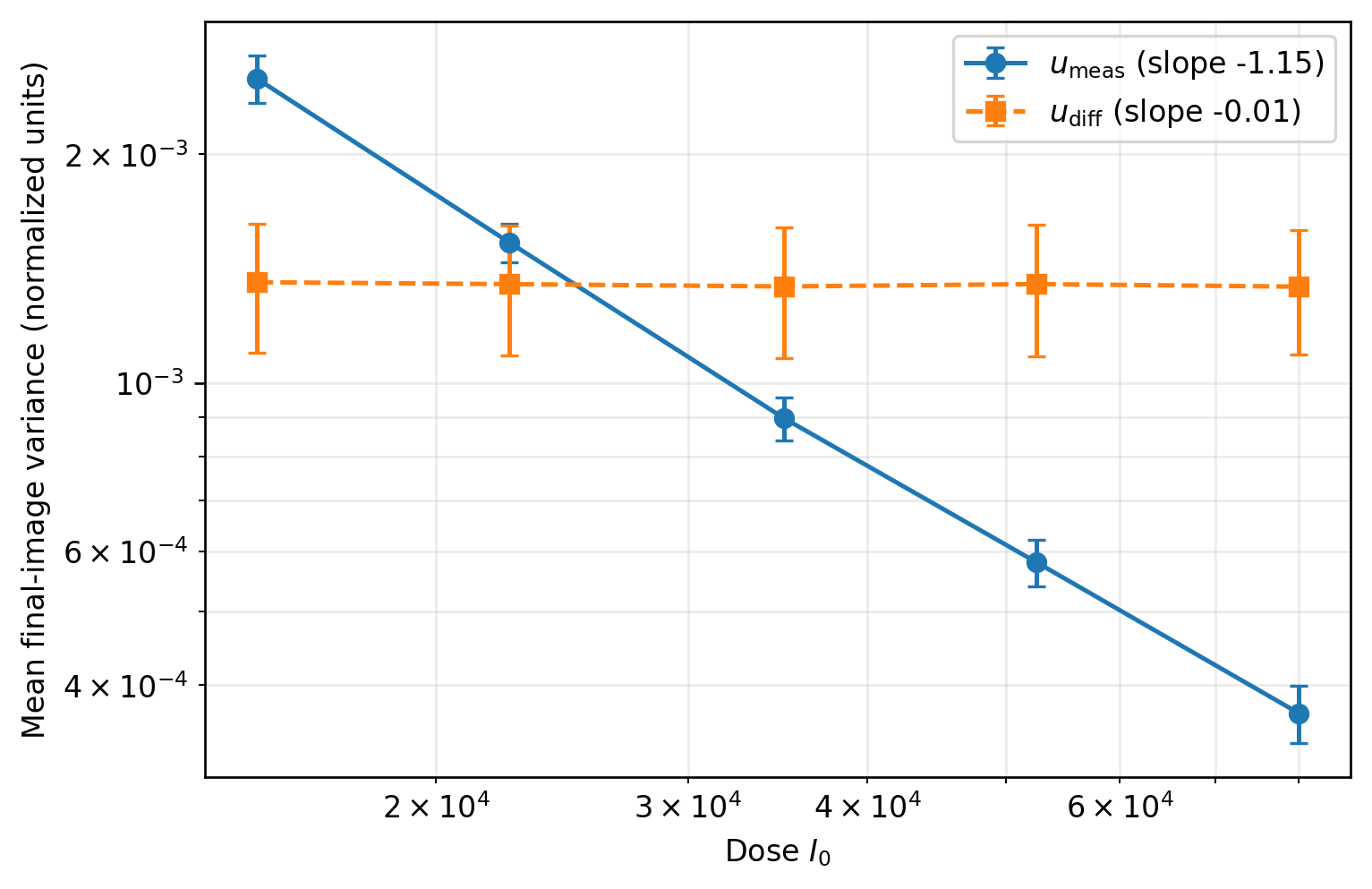}
\caption{Multi-dose stress test of the final-image variance decomposition using the frozen middle-dose DDPM. Points show means over eight held-out objects and error bars show between-object standard deviations. The fitted log--log slopes are $-1.15$ for $u_{\mathrm{meas}}$ and $-0.01$ for $u_{\mathrm{diff}}$.}
\label{fig:dosescaling}
\end{figure}

\begin{table}[h]
\centering
\caption{Measured final-image variance components across dose. Variances are in normalized image units and are averaged over eight held-out objects.}
\label{tab:dose}
\begin{tabular}{cccc}
\toprule
$I_0$ & mean $u_{\mathrm{meas}}$ & mean $u_{\mathrm{diff}}$ & scanner fraction\\
\midrule
$1.50\times10^4$ & $2.516\times10^{-3}$ & $1.357\times10^{-3}$ & $0.650$\\
$2.25\times10^4$ & $1.531\times10^{-3}$ & $1.349\times10^{-3}$ & $0.532$\\
$3.50\times10^4$ & $0.898\times10^{-3}$ & $1.339\times10^{-3}$ & $0.401$\\
$5.25\times10^4$ & $0.581\times10^{-3}$ & $1.349\times10^{-3}$ & $0.301$\\
$8.00\times10^4$ & $0.368\times10^{-3}$ & $1.339\times10^{-3}$ & $0.215$\\
\bottomrule
\end{tabular}
\end{table}

We also increased the middle-dose reference budget from $R_{\mathrm m}\times R_{\mathrm d}=24\times6$ to $96\times6$ for the eight held-out objects.  The median within-image split-half Spearman reliability increased from $0.449$ to $0.750$ (high-budget range $0.540$--$0.854$), while the fraction of raw negative scanner-component estimates decreased from $3.83\%$ to $1.05\%$.  The larger budget therefore improves reference stability, but it still does not meet the prospective $0.90$ criterion; reference-based spatial correlations remain exploratory.

\FloatBarrier
\subsection{Prospective full validation}
A definitive study will retrain the conditional DDPM jointly across the five dose levels using the planned $600/100/30$ training/validation/test object split and $T=200$ diffusion steps.  The nested reference budget will be increased adaptively until split-half reliability approaches $0.90$, and the principal uncertainty-prediction results will be repeated across multiple training seeds.  These experiments will distinguish full multi-dose predictive calibration from the physical-channel stress test reported above.

\FloatBarrier
\section{Discussion and Conclusion}
\label{sec:disc}

This pilot study provides direct support for the principle of PEEL-DDPM. Without any image-domain correction of the frozen DDPM output, the held-out Student-$t$ intervals were close to nominal calibration across four coverage levels, and the mean predicted variance closely matched the mean squared final-image residual. In addition, both $\hat V_{\rm pred}$ and the independently learned $\hat\uale$ stratified actual squared image error strongly.

The additional dose-stress experiment strengthens the physical interpretation of the decomposition: $u_{\mathrm{meas}}$ followed an approximately inverse-dose law (slope $-1.15$), whereas $u_{\mathrm{diff}}$ was essentially dose invariant (slope $-0.01$). Increasing the reference budget from $24\times6$ to $96\times6$ increased median within-image split-half reliability from $0.449$ to $0.750$, but the result remained below the pre-specified $0.90$ target.

The final-image endpoint also clarifies the physical decomposition. For a fixed object, repeated scans and repeated diffusion trajectories create repeated final reconstruction errors. Their total variance is exactly the final-image aleatoric coordinate under the PEEL operational definition, and the nested design separates what a new acquisition would change from what a new sampling trajectory would change. The first quantity is directly tied to measurement support, while the second reflects stochasticity of the reconstruction procedure and may be reduced by averaging or using a deterministic sampler.

The NIG identifiability argument remains unchanged in its essential form. The Student-$t$ likelihood fitted to the real image residual learns $(\alpha,c)$ but cannot determine $(\beta,\nu)$ separately. The nested physical experiment supplies the missing scalar $\uale$. Because the fiber is one-dimensional, this is exactly enough to recover a unique representative. The finer pair $(\umeas,\udiff)$ is measured directly and reported alongside the algebraically identified split.

The PEEL-DDPM formulation enjoys a clean decomposition. Student-$t$ NLL never trains the DDPM score location. The generator is optimized only by the standard $\varepsilon$-MSE objective and then frozen, while the evidential model sees its final image as a fixed location. Consequently, robust down-weighting in the Student-$t$ loss can change uncertainty parameters but cannot alter the generated image or the reverse score.

Several limitations remain. The full-path teacher is computationally expensive. The reported feasibility experiment and the frozen-model multi-dose stress test are still preliminary, so they can establish mechanism but not clinical calibration or robustness.  If the reconstruction has substantial systematic bias, interval centering can fail even when the variance is well learned. Equation~\eqref{eq:biasvariance}, the coverage test, and the location-calibration ablation can be used to expose failure modes. Finally, $\uepi$ remains an operational NIG remainder after the physically measured channel variance is removed.

In summary, PEEL-DDPM has been reformulated so that evidential learning answers the question that matters at the end of reconstruction: \emph{how uncertain is the generated image itself?} While standard DDPM training handles the internal noise-prediction problem, final-image Student-$t$ fitting models the real reconstruction error, and repeated physical and algorithmic channels identify the NIG uncertainty split. This PEEL-DDPM framework keeps image generation, error modeling, and physical identification distinct while integrating them into a unified approach for image reconstruction and uncertainty quantification.

\medskip
\paragraph{Author--AI Collaboration.}
Generative AI tools, including ChatGPT and Claude, were used to assist with formulation, drafting, simulation, and cross-checking. The author conceived the core idea and overall framework, directed and evaluated the AI-assisted work, and takes full responsibility for the content of this paper.

\end{document}